\documentclass[conference,letterpaper]{IEEEtran}

\usepackage[left = 0.74in, right=0.64in, bottom=0.98in, top =0.74in]{geometry}

\newcommand{\Latexfilepath}{.}
\usepackage{\Latexfilepath/Bermudez}
 \newcommand{\Figurespath}{\Latexfilepath/Figures}

\begin{document}

\title{Decentralized Machine Learning with Centralized Performance Guarantees via Gibbs Algorithms}
\IEEEoverridecommandlockouts

 \author{
   \IEEEauthorblockN{
   Yaiza Bermudez\IEEEauthorrefmark{1}, 
   Samir M. Perlaza\IEEEauthorrefmark{1}\IEEEauthorrefmark{2}\IEEEauthorrefmark{4}, and
   I\~{n}aki Esnaola\IEEEauthorrefmark{3}\IEEEauthorrefmark{4}\\
   Emails: name.lastname@inria.fr and esnaola@sheffield.ac.uk
                     }
   \IEEEauthorblockA{\IEEEauthorrefmark{1}%
                     Centre Inria d'Université Côte d'Azur, INRIA, 
                     Sophia Antipolis, France.}
   \IEEEauthorblockA{\IEEEauthorrefmark{2}%
                     Laboratoire GAATI, Universit\'{e} de la Polyn\'{e}sie fran\c{c}aise, Fa`a`\={a}, French Polynesia.}
   \IEEEauthorblockA{\IEEEauthorrefmark{3}%
                     School of Electrical and Electronic Engineering, University of Sheffield, Sheffield,
United Kingdom.}   
   \IEEEauthorblockA{\IEEEauthorrefmark{4}%
		     ECE Dept. Princeton University, Princeton, 08544 NJ, USA.
\thanks{This work is supported in part by the European Commission through the H2020-MSCA-RISE-2019 project 872172; the French National Agency for Research (ANR)  through the Project ANR-21-CE25-0013 and the project ANR-22-PEFT-0010 of the France 2030 program PEPR Réseaux du Futur; and in part by the Agence de l'innovation de défense (AID) through the project UK-FR 2024352. }                     } 
  }
 \maketitle 

\begin{abstract}
In this paper, it is shown, for the first time, that centralized performance is achievable in decentralized learning without sharing the local datasets. Specifically, when clients adopt an empirical risk minimization with relative-entropy regularization (ERM-RER) learning framework and a forward-backward communication between clients is established, it suffices to share the locally obtained Gibbs measures to achieve the same performance as that of a centralized ERM-RER with access to all the datasets.
The core idea is that the Gibbs measure produced by client~$k$ is used, as reference measure, by client~$k+1$. This effectively establishes a principled way to encode prior information through a reference measure.
In particular, achieving centralized performance in the decentralized setting requires a specific scaling of the regularization factors with the local sample sizes.
Overall, this result opens the door to novel decentralized learning paradigms that shift the collaboration strategy from sharing data to sharing the local inductive bias via the reference measures over the set of models.

\end{abstract}

\section{Introduction}
\enlargethispage{-0.10in}
Decentralized learning studies how a collection of clients can collaboratively tune a learning algorithm by communicating only over a network, without explicitly exchanging raw datasets. This setting extends early work on distributed and asynchronous optimization, where coordination is achieved through local computations and intermittent message passing \cite{tsitsiklis1986distributed,bertsekas1989parallel}. It becomes particularly relevant when a central coordinator is unavailable or undesirable, or when data transfers are impractical due to bandwidth, latency, ownership, privacy, or regulatory constraints \cite{kairouz2021advances}.
A standard benchmark for collaborative learning is the centralized regime in which all local datasets are pooled and a single training procedure is run on the aggregated data. While conceptually simple, the pooled-data benchmark is often unachievable in decentralized environments due to communication constraints and/or restricted disclosure of local datasets~\cite{kairouz2021advances,dwork2006dp}.
This benchmark is revisited through the lens of \emph{Gibbs algorithms}, i.e., data-dependent Gibbs probability measures on the model space. Such Gibbs measures naturally arise as solutions to empirical risk minimizations with relative-entropy regularization (ERM-RER) \cite{csiszar1975idivergence,perlaza2024empirical,daunas2024equivalence,daunas2025asymmetry}. This viewpoint also connects to exponential-weights predictors and PAC-Bayesian posteriors, which reason directly in terms of distributions on hypotheses \cite{cesabianchi2006prediction,mcallester1999some,seeger2002pacbayes,langford2002pacbayes,catoni2007pac,alquier2024pacbayes}. Beyond their variational interpretation, Gibbs measures also capture the long-run distribution of stochastic gradient methods under suitable regimes \cite{welling2011sgld,mandt2017sgd,raginsky2017nonconvex,Azizian2024What}.
From this standpoint, a complementary line of work studies Gibbs measures as solutions to ERM-RER problems and their extensions \cite{perlaza2024empirical,daunas2024equivalence,daunas2025asymmetry, bermudez2026machine}.
Other studies focus on change-of-measure techniques to quantify the variation of an expectation when the underlying probability measure changes \cite{csiszar1975idivergence,PerlazaEntropy2025}. These developments provide tools to interpret and manipulate Gibbs measures as first-class objects in learning systems, and to reason about how information is transported through probability measures rather than through datasets.

This paper shows that centralized performance guarantees can be achieved in a decentralized system through a strategic design of (i) the \emph{reference measures} and (ii) the \emph{regularization factors} that define the clients' Gibbs algorithm. More precisely, a peer-to-peer communication protocol is introduced, in which each client transmits its Gibbs probability measure to its successor, which adopts it as reference measure. This mechanism embeds information from datasets into the learning process without explicitly transmitting such datasets. A closed-form expression is obtained for the resulting decentralized Gibbs probability measures, together with conditions under which it coincides with the Gibbs measure induced by the centralized pooled-data benchmark.

The paper is organized as follows. Section~\ref{SupervisedMachineLearning} introduces the notation and formalizes the decentralized learning setting. Section~\ref{GibbsAlg} defines Gibbs conditional probability measures and their interpretation within the context of ERM-RER. Section~\ref{MainResult} presents the communication protocol and the main results establishing centralized-performance guarantees. Section~\ref{SecProofMainResult} sketches the proof of the main result. Section~\ref{Conclusion} concludes and discusses practical challenges, including the impact of distortions when probability measures are communicated under finite-rate constraints.
\enlargethispage{-0.10in}

\section{Supervised Machine Learning}\label{SupervisedMachineLearning}

Consider a decentralized learning system in which $K$ clients collaboratively tune their local learning algorithms by communicating with each other. For all $k \in \left\lbrace 1, 2, \ldots, K \right\rbrace$, let~$\set{M}_k$,~$\set{X}_k$ and~$\set{Y}_k$, with~$\set{M}_k \subseteq \reals^{d_k}$ and~$d_k \in \ints$, be sets of \emph{models}, \emph{patterns}, and \emph{labels}, respectively, at client~$k$.  
The training data available for client~$k$ consists of $n_k$ data points $(x_{k,1}, y_{k,1})$, $(x_{k,2}, y_{k,2})$, $\ldots$, $(x_{k,n_k}, y_{k,n_k})$, which are elements of the set $\set{Z}_k \triangleq \set{X}_k \times \set{Y}_k$. Such data points form the local training dataset, denoted by $\vect{z}_k \in \set{Z}_k^{n_k}$, which can be explicitly written as
\vspace{-1ex}
\begin{IEEEeqnarray}{rCl}
\label{EqDatasetK}
    \vect{z}_k & \triangleq & \autoparent{(x_{k,1}, y_{k,1}), (x_{k,2}, y_{k,2}), \ldots, (x_{k,n_k}, y_{k,n_k})}.
\end{IEEEeqnarray}
The dataset obtained by the aggregation of all local datasets, denoted by $\vect{z}_{0}$, satisfies
\begin{IEEEeqnarray}{rCl}
\label{EqDatasetZero}
\vect{z}_0 & \triangleq & \autoparent{\vect{z}_1, \vect{z}_2, \ldots, \vect{z}_K } \in \set{Z}_{1}^{n_1}\times\set{Z}_{2}^{n_2}\times\ldots\times\set{Z}_{k}^{n_K}\\
\label{EqDatasetZeroExplicit}
& = & \autoparent{(x_{0,1}, y_{0,1}), (x_{0,2}, y_{0,2}), \ldots, (x_{0,n_0}, y_{0,n_0})}.
\end{IEEEeqnarray}
Hence, the total number of data points, denoted by $n_0 \in \ints$, satisfies
$n_0  \triangleq \sum_{k =1}^{K} n_k$.
%
Given a model $\vect{\theta} \in \set{M}_k$ for client~$k$, the loss induced by such a model with respect to a data point~$\autoparent{x,y} \in \set{Z}_k$ is $\ell_k(x, y, \vect{\theta})$, where the function 
\begin{IEEEeqnarray}{rCl}
\label{EqLossFunction}
\ell_k: \set{Z}_k \times \set{M}_k\to [0, +\infty),
\end{IEEEeqnarray}
is referred to as the \emph{loss function} of client~$k$.
Such a loss function is assumed to be Borel measurable.
The \emph{empirical risk} induced by such a model~$\vect{\theta} \in \set{M}_k$, with respect to the dataset $\vect{z}_k$ in~\eqref{EqDatasetK}, is determined by the function
\begin{IEEEeqnarray}{rCl}
\label{EqEmpiricalRisk}
\mathsf{L}_k:  \function{\set{Z}_k^{n_k} \times \set{M}_k}{[0, +\infty)}{\autoparent{\vect{z}_k, \vect{\theta}}}{\frac{1}{n_k}\sum_{i=1}^{n_k}  \ell_k\left(x_{k,i}, y_{k,i}, \vect{\theta} \right),}
\end{IEEEeqnarray}
where the function $\ell_k$ is defined in \eqref{EqLossFunction}.

The set of all probability measures on the measurable space $\left(\set{M}_k,\mathscr{F}_{\set{M}_k}\right)$ is denoted by $\triangle\!\left(\set{M}_k,\mathscr{F}_{\set{M}_k}\right)$, or simply $\triangle\!\left(\set{M}_k\right)$. The set of all probability measures on $\set{M}_k$ conditioned on an element of ${\set{Z}_k^{n_k}}$ is denoted by $\simplex{\set{M}_k \mid \set{Z}_k^{n_k}}$. Moreover, the set of probability measures in $\simplex{\set{M}_k}$ that are absolutely continuous with respect to $Q_k$ is denoted by $\simplexabs{Q_k}{\set{M}_k}$. 
Using this notation,  a supervised machine learning algorithm is represented by a conditional probability measure, as defined hereunder. 
\begin{definition}[Algorithm]
For all $k \inCountK{K}$, a conditional probability measure $P_{\vect{\Theta}_k \mid \vect{Z}_k} \in \simplex{\set{M}_k \mid {\set{Z}_k^{n_k}}} $ is said to represent a supervised machine learning algorithm.
\end{definition}

Let $P_{\vect{\Theta}_k \mid \vect{Z}_k} \in \simplex{\set{M}_k \mid {\set{Z}_k^{n_k}}}$ be an algorithm. Hence, the
instance of such an algorithm trained upon the dataset $\vect{z}_k$ in \eqref{EqDatasetK} is denoted by $P_{\vect{\Theta}_k \mid \vect{Z}_k = \vect{z}_k}$, which is simply a probability measure in $\simplex{\set{M}_k}$.
\enlargethispage{-0.10in}
\section{Gibbs Algorithms}
\label{GibbsAlg}
The learning framework of client~$k$, with $k \inCountK{K}$, is defined by an  ERM-RER problem. To formalize this optimization problem, consider a dataset $\vect{z}_k \in \set{Z}_k^{n_k}$ and the functional $\mathsf{R}_{k,\vect{z}_k}$ defined as follows, 
\begin{IEEEeqnarray}{rCl}
\label{EqExpEmpRiskModels}
\mathsf{R}_{k,\vect{z}_k} : \function{ \simplex{\set{M}_k}}{[0, +\infty)}{P}{\int \mathsf{L}_k\autoparent{\vect{z}_k,\vect{\theta}} \mathrm{d}P\autoparent{\vect{\theta}},}
\end{IEEEeqnarray}
where the function $\mathsf{L}_k$ is defined in~\eqref{EqEmpiricalRisk}.
The corresponding ERM-RER problem is
\begin{IEEEeqnarray}{rCl}
\label{EqJul25at9h55in2025NiceA}
\min_{P \in \triangle_{Q_{k}} \autoparent{\set{M}_k}} & & \mathsf{R}_{k,\vect{z}_k} \left( P \right) + \lambda_k \KL{P}{Q_{k}},
\end{IEEEeqnarray}
where $Q_{k} \in \simplex{\set{M}_k}$ is a~$\sigma$-finite measure; $\lambda_k \in (0,+\infty)$ is the regularization factor; and $\KL{\cdot}{\cdot}$ represents the relative entropy,\cite[Definition~3]{PerlazaEntropy2025}. 
As shown later in Lemma~\ref{LemmaSolution}, the solution to~\eqref{EqJul25at9h55in2025NiceA}, whenever it exists, admits a closed-form expression.
This expression is a probability measure
parametrized by the empirical risk function $\mathsf{L}_k$; the~$\sigma$-finite measure~$Q_{k} \in \simplex{\set{M}_k}$; and the dataset~$\vect{z}_k \in \set{Z}_k^{n_k}$. Such  measures are referred to as Gibbs probability measures.
%
In order to define them, consider the following function: 
\begin{IEEEeqnarray}{cCl}
\label{EqJuly28at06h58in2025BusToSophia}
\mathsf{K}_{k, Q_{k}, \vect{z}_k} :\longfunction{ \reals}{\reals}{t}{\log \int \exp\left( t \, \mathsf{L}_k\autoparent{\vect{z}_k , \vect{\theta}_k}\right)\mathrm{d}Q_{k}\left( \vect{\theta}_k \right) ,} \middlesqueezeequ\spnum
\end{IEEEeqnarray}
where the function $\mathsf{L}_k$ is defined in~\eqref{EqEmpiricalRisk}.
Under the assumption that the reference measure~$Q_{k}$ is a probability measure, the function~$\mathsf{K}_{k, Q_{k}, \vect{z}_k}$ in~\eqref{EqJuly28at06h58in2025BusToSophia} is the cumulant generating function of the random variable~$\mathsf{L}_k\autoparent{\vect{z}_k , \vect{\theta}_k}$, for some fixed dataset~$\vect{z}_k \in \set{Z}_k^{n_k}$, when the model~$\vect{\theta}_k$ is sampled from $Q_{k}$.
Using this notation, the definition of the Gibbs conditional probability measure is presented hereunder.
\begin{definition} 
\label{DefGibbsModels}
Given the function~$\mathsf{L}_k$ in \eqref{EqEmpiricalRisk}; 
a~$\sigma$-finite measure~$Q_{k}$; and 
a~$\lambda_k \in (0, +\infty)$, with $k \inCountK{K}$,
the probability measure~$P^{(Q_{k}, \lambda_k)}_{\vect{\Theta}_k | \vect{Z}_k} \in \triangle\left(\set{M}_k | \set{Z}_k^{n_k} \right)$ is said to be an~$(\mathsf{L}_k, Q_{k}, \lambda_k)$-Gibbs conditional probability measure if
 \begin{IEEEeqnarray}{rCl}
 \mbox{$\forall \vect{z}_k \in \set{Z}^{n_k}$, } \mathsf{K}_{k, Q_{k}, \vect{z}_k} \left(\frac{-1}{\lambda_k} \right) < + \infty; 
 \end{IEEEeqnarray}
 and for all~$\left(\vect{z}_k, \vect{\theta}_k \right) \in \set{Z}^{n_k} \times \supp Q_{k}$,
\begin{IEEEeqnarray}{rcl}
\label{EqGibbsRNDFixedDataset}
\frac{\mathrm{d}P_{\vect{\Theta}_k|\vect{Z}_k = \vect{z}_k}^{(Q_{k}, \lambda_k)}}{\mathrm{d}Q_{k} } \autoparent{\vect{\theta}_k} & = & \exp\autoparent{\frac{-1}{\lambda_k} \EmpRisk{k}{\vect{z}_k}{\vect{\theta}_k} - \mathsf{K}_{k, Q_{k}, \vect{z}_k}\autoparent{\frac{-1}{\lambda_k}}}, \spnum
\end{IEEEeqnarray}
where the function~$ \mathsf{K}_{k, Q_{k}, \vect{z}_k}$ is defined in~\eqref{EqJuly28at06h58in2025BusToSophia}.%
\end{definition}
Note that, while $P_{\vect{\Theta}_k|\vect{Z}_k}^{(Q_{k}, \lambda_k)}$ in \eqref{EqGibbsRNDFixedDataset} is referred to as a Gibbs conditional probability measure, the measure $P_{\vect{\Theta}_k|\vect{Z}_k = \vect{z}_k}^{(Q_{k}, \lambda_k)}$, obtained by conditioning upon a given dataset~$\vect{z}_k\in\set{Z}_k^{n_k}$, is referred to as a Gibbs probability measure.
The following lemma formalizes the connection stated above between Gibbs measures and the ERM-RER problem in~\eqref{EqJul25at9h55in2025NiceA}.
\begin{lemma}\label{LemmaSolution}
Assume that the optimization problem in \eqref{EqJul25at9h55in2025NiceA}  admits a solution. Then, the $(\mathsf{L}_k, Q_{k}, \lambda_k)$-Gibbs probability measure $P_{\vect{\Theta}_k|\vect{Z}_k = \vect{z}_k}^{(Q_{k}, \lambda_k)}$ in~\eqref{EqGibbsRNDFixedDataset} is the unique solution. 
\end{lemma}
\begin{IEEEproof}
The proof follows from \cite[Lemma~1]{PerlazaEntropy2025}.
\end{IEEEproof}
\enlargethispage{-0.10in}
This result has also been reported for other $f$-divergences in~\cite{daunas2024equivalence, daunas2025asymmetry}.
Interestingly, the probability measure $P_{\vect{\Theta}_k|\vect{Z}_k= \vect{z}_k}^{(Q_{k}, \lambda_k)}$ in~\eqref{EqGibbsRNDFixedDataset} is the long-run distribution of a stochastic gradient descent algorithm~\cite{Azizian2024What}. In statistical learning, such a distribution is often referred to as the \emph{Gibbs algorithm} \cite{perlaza2023validation}. 

Another optimization problem that is closely related to the probability measure $P_{\vect{\Theta}_k|\vect{Z}_k = \vect{z}_k}^{(Q_{k}, \lambda_k)}$ in~\eqref{EqGibbsRNDFixedDataset} is the following:
\begin{subequations}\label{EqJanuary1at13h54in2026atNice}
\begin{IEEEeqnarray}{rCl}
\min_{P \in \triangle_{Q_{k}} \autoparent{\set{M}_k}} & & \mathsf{R}_{k,\vect{z}_k} \left( P \right) \\
\text{s.t.} & \quad &  \KL{P}{Q_{k}} \leqslant \gamma_k,
\end{IEEEeqnarray} 
\end{subequations}
for some $\gamma_k>0$.
The following lemma establishes the connection. 
\begin{lemma}\label{LemmaJanuary1at14h24in2026}
Assume that the optimization problem in \eqref{EqJanuary1at13h54in2026atNice} admits a solution and that $\lambda_k$ is such that 
\begin{IEEEeqnarray}{rCl}
 \KL{P_{\vect{\Theta}_k|\vect{Z}_k = \vect{z}_k}^{(Q_{k}, \lambda_k)}}{Q_{k}} & = & \gamma_k.
\end{IEEEeqnarray}
Then, the $(\mathsf{L}_k, Q_{k}, \lambda_k)$-Gibbs probability measure $P_{\vect{\Theta}_k|\vect{Z}_k = \vect{z}_k}^{(Q_{k}, \lambda_k)}$ in~\eqref{EqGibbsRNDFixedDataset} is the unique solution to~\eqref{EqJanuary1at13h54in2026atNice}.
\end{lemma}
\begin{IEEEproof}
The proof follows from \cite[Lemma~4]{PerlazaEntropy2025}.
\end{IEEEproof}
Lemma~\ref{LemmaJanuary1at14h24in2026} implies that the probability measure $P_{\vect{\Theta}_k|\vect{Z}_k = \vect{z}_k}^{(Q_{k}, \lambda_k)}$ in~\eqref{EqGibbsRNDFixedDataset} minimizes the training empirical risk over all probability measures in the following neighborhood of~$Q_k$,
\begin{IEEEeqnarray}{rCl}
\label{EqJanuary1at19h39in2026HomeNice}
\left\lbrace P \in \simplexabs{Q_{k}}{\set{M}_k} :  \KL{P}{Q_{k}} \leqslant \gamma_k \right\rbrace.
\end{IEEEeqnarray}
This observation is important for presenting the main results.

\section{Main result}\label{MainResult}
The main result of this work (Theorem~\ref{TheoremDec21at18h00in2025inMadrid}) is presented in Subsection~\ref{SubsecMain}. In order to present such a result, the peer-to-peer communication protocol used by the clients is introduced. The section ends by stating the necessary conditions under which centralized performance is obtained.

\subsection{Communication Protocol}
\label{SubsecComm}
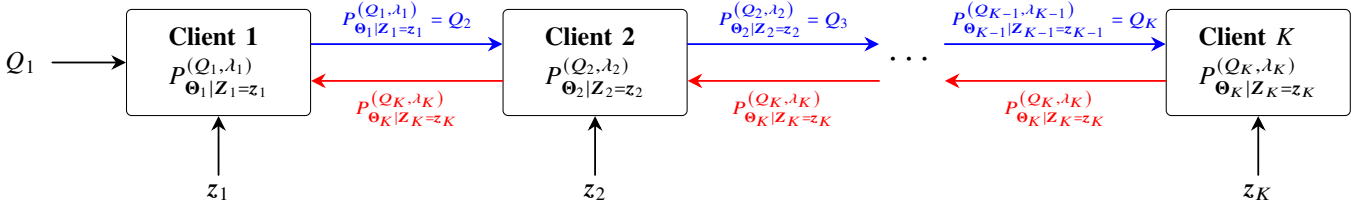
\begin{figure*}[t]
\centering
\resizebox{\textwidth}{!}{%
  \begin{tikzpicture}[
  line cap=round, line join=round,
  arr/.style={-{Stealth[length=2.0mm,width=2.0mm]}, line width=0.7pt},
  msgarr/.style={-{Stealth[length=2.0mm,width=2.0mm]}, line width=0.7pt, blue},
  backarr/.style={-{Stealth[length=2.0mm,width=2.0mm]}, line width=0.7pt, red}, 
  node distance=18mm and 26mm,
  prior/.style={align=center, inner sep=6pt},
  client/.style={draw, rounded corners=2pt, align=left, inner sep=7pt, minimum width=2.5cm},
  data/.style={text depth=0.50ex},
  msg/.style={midway, above, align=center, font=\scriptsize, text depth=0.30ex, blue},
  msgback/.style={midway, below, align=center, font=\scriptsize, text depth=0.30ex, red},
]

\node[prior] (Q1) {$Q_1$};

\node[client, right=10mm of Q1] (C1) {%
  \textbf{Client 1}\\[1mm]
  $
  P^{\autoparent{Q_1,\lambda_1}}_{\vect{\Theta}_1\mid \vect{Z}_1=\vect{z}_1}$};

\node[client, right=26mm of C1] (C2) {%
  \textbf{Client 2}\\[1mm]
  $
  P^{\autoparent{Q_2,\lambda_2}}_{\vect{\Theta}_2\mid \vect{Z}_2=\vect{z}_2}$};

\node[right=26mm of C2, font=\Large] (dots) {$\cdots$};

\node[client, right=30mm of dots] (CK) {%
  \textbf{Client $K$}\\[1mm]
  $P^{\autoparent{Q_K,\lambda_K}}_{\vect{\Theta}_K\mid \vect{Z}_K=\vect{z}_K}$};

\node[data, below=8mm of C1] (Z1) {$\vect{z}_1$};
\node[data, below=8mm of C2] (Z2) {$\vect{z}_2$};
\node[data, below=8mm of CK] (ZK) {$\vect{z}_K$};

\draw[arr] (Z1) -- (C1.south);
\draw[arr] (Z2) -- (C2.south);
\draw[arr] (ZK) -- (CK.south);

\draw[arr] (Q1) -- (C1);

\draw[msgarr]  ([yshift=2.5mm]C1.east) --([yshift=2.5mm]C2.west)
node[msg] {$
P^{\autoparent{Q_1,\lambda_1}}_{\vect{\Theta}_1\mid \vect{Z}_1=\vect{z}_1}=Q_2
$\\[-2mm]};

\draw[msgarr]   ([yshift=2.5mm]C2.east)
 --  ([yshift=2.5mm]dots.west)
node[msg] {$
P^{\autoparent{Q_2,\lambda_2}}_{\vect{\Theta}_2\mid \vect{Z}_2=\vect{z}_2}=Q_3
$\\[-2mm]};

\draw[msgarr] ([yshift=2.5mm]dots.east) -- ([yshift=2.5mm]CK.west) 
node[msg] {$
P^{\autoparent{Q_{K-1},\lambda_{K-1}}}_{\vect{\Theta}_{K-1}\mid \vect{Z}_{K-1}=\vect{z}_{K-1}}=Q_K
$\\[-2mm]};

\draw[backarr] ([yshift=-2.5mm]C2.west) -- ([yshift=-2.5mm]C1.east)
node[msgback] {$P^{\autoparent{Q_K,\lambda_K}}_{\vect{\Theta}_K\mid \vect{Z}_K=\vect{z}_K}$};

\draw[backarr] ([yshift=-2.5mm]dots.west) -- ([yshift=-2.5mm]C2.east)
node[msgback] {$P^{\autoparent{Q_K,\lambda_K}}_{\vect{\Theta}_K\mid \vect{Z}_K=\vect{z}_K}$};

\draw[backarr] ([yshift=-2.5mm]CK.west) -- ([yshift=-2.5mm]dots.east)
node[msgback] {$P^{\autoparent{Q_K,\lambda_K}}_{\vect{\Theta}_K\mid \vect{Z}_K=\vect{z}_K}$};

\end{tikzpicture}}
\caption{Nested Structure: the Gibbs measure produced by client $k$ becomes the reference measure $Q_{k+1}$ used by client $k+1$.}
\vspace{-2ex}
\label{FigNestedStructure}
\end{figure*}
\enlargethispage{-0.10in}
Figure~\ref{FigNestedStructure} depicts the forward--backward peer-to-peer communication protocol used in this work.
In the forward direction (blue arrows), for all $k \in \{1,2,\ldots,K-1\}$, client~$k$ transmits to client~$k+1$ the $(\mathsf{L}_k,Q_k,\lambda_k)$-Gibbs probability measure $ P^{\autoparent{Q_{k}, \lambda_k}}_{\vect{\Theta}_k \mid \vect{Z}_k = \vect{z}_{k}}$ in~\eqref{EqGibbsRNDFixedDataset}, obtained as the solution to the ERM-RER problem in \eqref{EqJul25at9h55in2025NiceA}.
Client~$k+1$ adopts this transmitted measure as its reference measure~$Q_{k+1}$ in \eqref{EqJul25at9h55in2025NiceA}.
%
This choice induces a nested structure of the reference measure, as $
Q_{k+1} = P^{\autoparent{Q_{k}, \lambda_k}}_{\vect{\Theta}_k \mid \vect{Z}_k = \vect{z}_{k}}$,  with $k \in \{1,2,\ldots,K-1\}$.
This is formalized later in Theorem~\ref{TheoremDec21at18h00in2025inMadrid}.
The backward direction (red arrows) disseminates the final Gibbs probability measure.
More specifically, once client~$K$ has computed its $(\mathsf{L}_K,Q_K,\lambda_K)$-Gibbs probability measure $ P^{\autoparent{Q_{K}, \lambda_K}}_{\vect{\Theta}_K \mid \vect{Z}_K = \vect{z}_{K}}$  in~\eqref{EqGibbsRNDFixedDataset}, this measure is transmitted back along the chain, from client~$K$ to client~$K-1$, then from client~$K-1$ to client~$K-2$, and so on until client~$1$.
This backward transmission provides all clients with access to the same final Gibbs algorithm.
The following subsection describes such a final algorithm.

\subsection{Decentralized Algorithms}
\label{SubsecMain}
 The main result of this work is presented by the following theorem.  
\begin{theorem}
\label{TheoremDec21at18h00in2025inMadrid}
For all $k \inCountK{K}$, consider an $(\mathsf{L}_k,Q_{k},\lambda_k)$-Gibbs conditional probability measure, denoted by $P^{\autoparent{Q_{k}, \lambda_k}}_{\vect{\Theta}_k | \vect{Z}_k} \in \simplex{\set{M}_k | \set{Z}^{n_k}_k}$, where the reference measure $Q_{k}$
satisfies 
\begin{IEEEeqnarray}{rcl}
\label{EqDec27at10h57in2025inMadrid}
Q_k &=& \left\lbrace
\begin{array}{cc}
Q_1 & \text{if } k = 1\\
P^{\autoparent{Q_{k-1},\lambda_{k-1} }}_{\vect{\Theta}_{k-1} | \vect{Z}_{k-1} = \vect{z}_{k-1}}  & \text{if } k \geqslant 2,
\end{array}
\right.
\end{IEEEeqnarray}
for some given $Q_1$. Then, for all $\vect{\theta} \in \supp Q_1$, 
\begin{IEEEeqnarray}{rCl}
\label{EqDec30at11h25in2025Antibes}
\RND{P^{\autoparent{Q_{k}, \lambda_k}}_{\vect{\Theta}_k | \vect{Z}_k = \vect{z}_{k}}}{Q_1}\autoparent{\vect{\theta}} &=& \frac{\exp\autoparent{\sum_{j=1}^{k} \frac{-1}{\lambda_j} \EmpRiskk{\vect{z}}{\vect{\theta}}{j}}}
{\int \exp\autoparent{\sum_{i=1}^{k} \frac{-1}{\lambda_i} \EmpRiskk{\vect{z}}{\vect{\nu}}{i}}\mathrm{d} Q_1(\vect{\nu})}.
\end{IEEEeqnarray}
\end{theorem}
\begin{IEEEproof}
The proof is presented in \cite{InriaRR9608}.
\end{IEEEproof}

The choice of reference measures $Q_1$, $Q_2$, $\ldots$, $Q_K$ in~\eqref{EqDec27at10h57in2025inMadrid} induces a \emph{nested} structure. Under this structure, the training performed by client~$k$ uses only its local dataset, while the influence of the previous clients' datasets is carried out through the reference measure $Q_k$.
The relevance of this nested structure, in which client $k$ shares its Gibbs probability measure (algorithm) with its successor, client $k+1$, is made clear by Lemma~\ref{LemmaSolution}. 
More specifically, the probability measure
$P^{\autoparent{Q_{k}, \lambda_k}}_{\vect{\Theta}_k \mid \vect{Z}_k = \vect{z}_{k}}$ in \eqref{EqDec30at11h25in2025Antibes}
is the unique solution to \eqref{EqJul25at9h55in2025NiceA} and, simultaneously, the unique minimizer of the following optimization problem:
\begin{IEEEeqnarray}{rCl}
\label{EqDec28at23h24in2025FlyingToNice}
\min_{P \in \simplexabs{Q_1}{\set{M}_k}}
\int \autoparent{\sum_{j=1}^{k} \frac{1}{\lambda_j} \EmpRiskk{\vect{z}}{\vect{\theta}}{j}}
\,\mathrm{d}P\autoparent{\vect{\theta}} + \KL{P}{Q_1}.\spnum
\end{IEEEeqnarray}
The following corollary of Theorem~\ref{TheoremDec21at18h00in2025inMadrid} formalizes this observation.
\begin{corollary}\label{CorJanuary1at20h49in2026HomeNice}
Under the assumption that the measures $Q_1$, $Q_2$, $\ldots$, $Q_K$ satisfy~\eqref{EqDec27at10h57in2025inMadrid}, the solutions to the optimization problems in~\eqref{EqJul25at9h55in2025NiceA} and~\eqref{EqDec28at23h24in2025FlyingToNice} are unique and coincide with the $(\mathsf{L}_k, Q_{k}, \lambda_k)$-Gibbs probability measure $P^{\autoparent{Q_{k}, \lambda_k}}_{\vect{\Theta}_k \mid \vect{Z}_k = \vect{z}_{k}}$ in~\eqref{EqDec30at11h25in2025Antibes}.
\end{corollary}
\enlargethispage{-0.10in}
Given $k > 1$, the optimization problem in~\eqref{EqJul25at9h55in2025NiceA} depends exclusively on the local training dataset $\vect{z}_k$. While the reference $Q_k$ in~\eqref{EqDec27at10h57in2025inMadrid} depends on the training datasets of the $k-1$ previous clients, such datasets do not need to be explicitly known for solving~\eqref{EqJul25at9h55in2025NiceA}. In particular, solving~\eqref{EqJul25at9h55in2025NiceA} requires access to the probability measure $Q_k \in \simplex{\set{M}_k}$, but not access to the training datasets of all previous clients. This is because for fixed training datasets, $Q_k$ is simply a probability measure on the model space.
In contrast, solving the optimization problem in~\eqref{EqDec28at23h24in2025FlyingToNice} requires knowing the training datasets of client~$j$, for all $j \inCountK{k}$. In this case, the reference measure,  $Q_1$,  does not depend on any training dataset. 
The fact that problems~\eqref{EqJul25at9h55in2025NiceA} and~\eqref{EqDec28at23h24in2025FlyingToNice} share the same solution unveils an important observation: providing client~$k$ with a reference measure $Q_k$ of the form in~\eqref{EqDec27at10h57in2025inMadrid} reproduces the effect of having access to the training datasets of the $k-1$ previous clients.
The following subsection unveils, under specific conditions, the centralized-type guarantees of the nested structure.

\subsection{Centralized Performance Guarantees}
An important observation is that a strategic choice of $\lambda_1$, $\lambda_2$, $\ldots$, $\lambda_K$ in~\eqref{EqDec30at11h25in2025Antibes} can lead to achieving the same Gibbs probability distribution as in a setting in which the training datasets of all clients are available to all clients. 
This describes a decentralized system whose distributed nature does not prevent it from achieving the same Gibbs algorithm that would have been obtained if all the training datasets were available to all clients.
The following theorem formalizes this observation.
%
\begin{theorem}
\label{EqDec30at21h00in2025inAntibes}
Assume that the loss functions in \eqref{EqLossFunction} satisfy $\ell_1=\ell_2 = \ldots = \ell_K= \ell$ and for all $k \in \{1, \ldots, K\}$, 
\begin{IEEEeqnarray}{rcl}
\label{EqJanuary1at15h16in2026atHomeNice}
\lambda_k  = \frac{n_0\lambda_0}{n_k},
\end{IEEEeqnarray}
for some $\lambda_0 > 0$ and some loss function $\ell$.
Consider some measures $Q_1$, $Q_2$, $\cdots$, $Q_K$ satisfying~\eqref{EqDec27at10h57in2025inMadrid}.
Then, for all $\vect{\theta}\in\supp Q_1$, it follows that,
\begin{IEEEeqnarray}{rCl}
\label{EqJanuary1at16h30in2026atHomeNice}
\RND{P^{\autoparent{Q_{K}, \lambda_K}}_{\vect{\Theta}_K | \vect{Z}_K = \vect{z}_{K}}}{P^{\autoparent{Q_1, \lambda_0}}_{\vect{\Theta}_K | \vect{Z}_0 = \vect{z}_{0}}}\autoparent{\vect{\theta}} &=&1,
\end{IEEEeqnarray}
where the probability measure $P^{\autoparent{Q_{K}, \lambda_K}}_{\vect{\Theta}_K | \vect{Z}_K = \vect{z}_{K}}$ is defined in \eqref{EqDec30at11h25in2025Antibes}; the measure  $P^{\autoparent{Q_1, \lambda_0}}_{\vect{\Theta}_K | \vect{Z}_0 = \vect{z}_{0}}$ satisfies for all $\vect{\theta}\in\supp Q_1$, 
\begin{IEEEeqnarray}{rCl}
\label{EqDec30at15h43in2025Antibes}
\RND{P^{\autoparent{Q_{1}, \lambda_0}}_{\vect{\Theta}_K | \vect{Z}_0 = \vect{z}_{0}}}{Q_1}\autoparent{\vect{\theta}} &=& \frac{\exp\autoparent{\frac{-1}{n_0\lambda_0}\sum_{i=1}^{n_0}  \ell\left(x_{0,i}, y_{0,i}, \vect{\theta} \right)}}{\int \exp\autoparent{\frac{-1}{n_0\lambda_0}\sum_{i=1}^{n_0}  \ell\left(x_{0,i}, y_{0,i}, \vect{\nu} \right)}\mathrm{d} Q_1(\vect{\nu})}; \squeezeequ\spnum
\end{IEEEeqnarray}
and $\left(x_{0,i}, y_{0,i}\right)$ are data points of the aggregated dataset $\vect{z}_0$ in~\eqref{EqDatasetZero}.
\end{theorem}
\begin{IEEEproof}
The proof is presented in \cite{InriaRR9608}.
\end{IEEEproof}
The relevance of Theorem~\ref{EqDec30at21h00in2025inAntibes} is highlighted by the following observations. 
\enlargethispage{-0.10in}
Under the assumptions of Theorem~\ref{EqDec30at21h00in2025inAntibes}, in particular that the loss functions satisfy $\ell_1=\ell_2=\cdots=\ell_K=\ell$, the equality in~\eqref{EqJanuary1at15h16in2026atHomeNice} together with~\cite[Lemma~4]{perlaza2023validation} allows rewriting the optimization problem in~\eqref{EqDec28at23h24in2025FlyingToNice} as
\begin{IEEEeqnarray}{rCl}
\label{EqJanuary1at15h20in2026atHomeNice}
\min_{P \in \simplexabs{Q_1}{\set{M}_K}}
\int\frac{1}{n_0}\sum_{i=1}^{n_0}  \ell\left(x_{0,i}, y_{0,i}, \vect{\theta} \right)
\,\mathrm{d}P\autoparent{\vect{\theta}} + \lambda_0\KL{P}{Q_1},\middlesqueezeequ\spnum
\end{IEEEeqnarray}
which requires access to the training datasets of all clients.
%
%
Interestingly, from Lemma~\ref{LemmaSolution}, it follows that the probability measure $P^{\autoparent{Q_1, \lambda_0}}_{\vect{\Theta}_K | \vect{Z}_0 = \vect{z}_{0}}$ in \eqref{EqDec30at15h43in2025Antibes} is the solution to~\eqref{EqJanuary1at15h20in2026atHomeNice}.
More importantly, from Lemma~\ref{LemmaJanuary1at14h24in2026}, if $\lambda_0$ is chosen such that 
\vspace{-1ex}
\begin{IEEEeqnarray}{rCl}
 \KL{P^{\autoparent{Q_{1}, \lambda_0}}_{\vect{\Theta}_K | \vect{Z}_0 = \vect{z}_{0}}}{Q_{1}} & = & \gamma_0,
 \vspace{-1ex}
\end{IEEEeqnarray}
for some $\gamma_0 > 0$, the probability measure $P^{\autoparent{Q_1, \lambda_0}}_{\vect{\Theta}_K | \vect{Z}_0 = \vect{z}_{0}}$ in~\eqref{EqDec30at15h43in2025Antibes} is also the solution to the following optimization problem:
\begin{subequations}
\vspace{-1.5ex}
\label{EqJanuary1at15h37in2026atHomeNice}
\begin{IEEEeqnarray}{rCl}
\min_{P \in \simplexabs{Q_1}{\set{M}_K}} & &
\int\frac{1}{n_0}\sum_{i=1}^{n_0}  \ell\left(x_{0,i}, y_{0,i}, \vect{\theta} \right)\,\mathrm{d}P\autoparent{\vect{\theta}} \\
\text{s.t.} &\quad& \KL{P}{Q_1} \leqslant \gamma_0.
\end{IEEEeqnarray}
\end{subequations}
From \eqref{EqJanuary1at16h30in2026atHomeNice}, it follows that the measures $P^{\autoparent{Q_{K}, \lambda_K}}_{\vect{\Theta}_K | \vect{Z}_K = \vect{z}_{K}}$ and $P^{\autoparent{Q_1, \lambda_0}}_{\vect{\Theta}_K | \vect{Z}_0 = \vect{z}_{0}}$ are identical, which implies that the nested structure, induced by the choice of reference measures in~\eqref{EqDec27at10h57in2025inMadrid} and the regularization factors in \eqref{EqJanuary1at15h16in2026atHomeNice},  allows achieving in a decentralized system, the same learning algorithm that would have been obtained in a centralized system in which all training datasets are available to all clients.

Under the forward--backward communication protocol in Section~\ref{SubsecComm}, after $K-1$ forward messages (blue arrows in Figure~\ref{FigNestedStructure}) client~$K$ obtains the probability measure that minimizes the empirical risk with respect to all training datasets within the neighborhood of~$Q_1$. The backward dissemination (red arrows in Figure~\ref{FigNestedStructure}) provides each client with the same Gibbs algorithm, minimizing within a neighborhood of the form in~\eqref{EqJanuary1at19h39in2026HomeNice} around~$Q_1$ the empirical risk with respect to the aggregated dataset.



\section{Proof of Main Result}
\label{SecProofMainResult}
This section first introduces the preliminaries needed for the proof of Theorem~\ref{TheoremDec21at18h00in2025inMadrid}, and then outlines the main steps. A more detailed proof appears in \cite{InriaRR9608}.

\subsection{Preliminaries}	
\enlargethispage{-0.20in}
\begin{lemma}\label{LemmaGeneralNiceResult}
For all $k \inCountK{K}$, consider an $(\mathsf{L}_k,Q_{k},\lambda_k)$-Gibbs conditional probability measure, denoted by $P^{\autoparent{Q_{k}, \lambda_k}}_{\vect{\Theta}_k | \vect{Z}_k} \in \simplex{\set{M}_k | \set{Z}^{n_k}_k}$, where the reference measure $Q_{k}$
satisfies \eqref{EqDec27at10h57in2025inMadrid}.
Then, for all $\vect{\theta} \in \supp Q_1$, 
\begin{IEEEeqnarray}{rCl}
\label{EqDec15at21h43in2025inMadridB}
\RND{P^{\autoparent{Q_{k}, \lambda_{k}}}_{\vect{\Theta}_k | \vect{Z}_k = \vect{z}_{k}}}{Q_{k}}\autoparent{\vect{\theta}}
&=& \RND{P^{\autoparent{Q_1, \lambda_k}}_{\vect{\Theta}_k | \vect{Z}_k = \vect{z}_{k}}}{Q_1}\autoparent{\vect{\theta}} \exp \autoparent{C_{k}}, 
\end{IEEEeqnarray}
where the measure $P^{\autoparent{Q_1, \lambda_k}}_{\vect{\Theta}_k | \vect{Z}_k }$ is an $(\mathsf{L}_k, Q_1, \lambda_k)$-Gibbs conditional probability measure and $C_{k} \in \mathbb{R}$ satisfies
\begin{IEEEeqnarray}{rCl}
\label{EqDec27at11h22in2025inMadrid}
\nonumber
&&C_{k}
\triangleq\\
&&\log \autoparent{\frac{ \exp\autoparent{\mathsf{K}_{k,Q_1,\vect{z}_{k}}\autoparent{\frac{-1}{\lambda_{k}}}} \int \exp\autoparent{-\displaystyle\sum_{i=1}^{k-1}~\frac{1}{\lambda_i}~\EmpRiskk{\vect{z}}{\vect{\nu}}{i}}\mathrm{d} Q_1(\vect{\nu})
}{\int \exp\autoparent{ -\sum_{j=1}^{k}~\frac{1}{\lambda_j}~\EmpRiskk{\vect{z}}{\vect{\nu}^\prime}{j}} \mathrm{d} Q_1(\vect{\nu}^\prime)}},\middlesqueezeequ \spnum
\end{IEEEeqnarray}
where the functional $\mathsf{K}_{k,Q_1,\vect{z}_{k}}$ is defined in \eqref{EqJuly28at06h58in2025BusToSophia}. 
\end{lemma}
%
\begin{IEEEproof}
The proof is presented in \cite{InriaRR9608}. 
\end{IEEEproof}
The relevance of this lemma lies in the fact that the Radon--Nikodym derivative of an $(\mathsf{L}_k,Q_k,\lambda_k)$-Gibbs conditional probability measure with respect to its reference measure $Q_k$ can be re-expressed relative to the fixed reference measure $Q_1$, up to a factor, $C_{k}$ in \eqref{EqDec27at11h22in2025inMadrid}, that does not depend on the model. Moreover, this factor admits an information-theoretic characterization in terms of Kullback--Leibler divergences; see \cite[Lemma~9]{InriaRR9608}. Lemma~\ref{LemmaGeneralNiceResult} constitutes a main step in the proof of Theorem~\ref{TheoremDec21at18h00in2025inMadrid} and is explained in the following subsection.

\subsection{Sketched proof of Theorem~\ref{TheoremDec21at18h00in2025inMadrid}}
Using \cite[Theorem~4]{InriaRR9591} (chain rule), the Radon--Nikodym derivative in the left-hand side of \eqref{EqDec15at21h43in2025inMadridB} can be written as follows
\vspace{-2ex}
\begin{IEEEeqnarray}{rcl}
\label{EqJan5at11h58in2026inSophia}
\RND{P^{\autoparent{Q_{k}, \lambda_{k}}}_{\vect{\Theta}_k | \vect{Z}_k = \vect{z}_{k}}}{Q_{k}}\autoparent{\vect{\theta}}
&=& \RND{P^{\autoparent{Q_{k}, \lambda_{k}}}_{\vect{\Theta}_k | \vect{Z}_k = \vect{z}_{k}}}{Q_{1}}\autoparent{\vect{\theta}}
\RND{Q_{1}}{Q_{k}}\autoparent{\vect{\theta}},
\end{IEEEeqnarray}
Then, from Lemma~\ref{LemmaGeneralNiceResult} and \cite[Theorem~5]{InriaRR9591}, the equality in \eqref{EqJan5at11h58in2026inSophia} yields
\begin{IEEEeqnarray}{rcl}
\label{EqJan5at11h59in2026inSophia}
 \RND{P^{\autoparent{Q_{k}, \lambda_{k}}}_{\vect{\Theta}_k | \vect{Z}_k = \vect{z}_{k}}}{Q_{1}}\autoparent{\vect{\theta}} &=&
\RND{Q_{k}}{Q_{1}}\autoparent{\vect{\theta}}  \RND{P^{\autoparent{Q_1, \lambda_k}}_{\vect{\Theta}_k | \vect{Z}_k = \vect{z}_{k}}}{Q_1}\autoparent{\vect{\theta}} \exp \autoparent{C_{k}}.
\end{IEEEeqnarray}
Moreover, from \cite[Lemma~3]{perlaza2024empirical}, the Radon--Nikodym derivatives in~\eqref{EqJan5at11h58in2026inSophia} and in~\eqref{EqJan5at11h59in2026inSophia} are well defined. 

\enlargethispage{-0.10in}
The next step consists in expressing $\RND{Q_k}{Q_1}$ using the nested definition of the reference measures in~\eqref{EqDec27at10h57in2025inMadrid}. This nested construction implies the absolute continuity assumptions required to apply a chain rule for Radon--Nikodym derivatives. The resulting product decomposition isolates successive changes of reference measure and therefore allows each factor to be handled separately using Lemma~\ref{LemmaGeneralNiceResult}.
 Then, from \cite[Theorem~4]{InriaRR9591}, it follows that
\begin{IEEEeqnarray}{rCl}
\label{EqJan5qt13h51in2026inSophiaA}
\RND{Q_k}{Q_1}\autoparent{\vect{\theta}}
&=&
\prod_{j=2}^{k} \RND{Q_j}{Q_{j-1}}\autoparent{\vect{\theta}}\\
\label{EqJan5qt13h51in2026inSophiaB}
&=&
\RND{Q_2}{Q_{1}}\autoparent{\vect{\theta}}\prod_{j=3}^{k} \RND{Q_j}{Q_{j-1}}\autoparent{\vect{\theta}}\\
\label{EqJan5qt13h51in2026inSophiaC}
&=& \RND{P^{\autoparent{Q_1, \lambda_1}}_{\vect{\Theta}_1 \mid \vect{Z}_1 = \vect{z}_1}}{Q_{1}}\autoparent{\vect{\theta}} \prod_{j=2}^{k-1} \RND{P^{\autoparent{Q_j, \lambda_j}}_{\vect{\Theta}_j \mid \vect{Z}_j = \vect{z}_j}}{Q_{j}}\autoparent{\vect{\theta}}\\
\label{EqJan5qt13h51in2026inSophiaD}
&=& \frac{\exp\autoparent{ -\sum_{i=1}^{k-1} \frac{1}{\lambda_i} \EmpRisk{i}{\vect{z}_i}{\vect{\theta}}}}{\int \exp\autoparent{ -\sum_{i=1}^{k-1} \frac{1}{\lambda_i} \EmpRisk{i}{\vect{z}_i}{\vect{\nu}}} \mathrm{d}Q_{1}\autoparent{\vect{\nu}}},
\end{IEEEeqnarray}
where \eqref{EqJan5qt13h51in2026inSophiaC} follows from \eqref{EqDec27at10h57in2025inMadrid}; and  \eqref{EqJan5qt13h51in2026inSophiaD} follows from Definition~\ref{DefGibbsModels} and Lemma~\ref{LemmaGeneralNiceResult}. Substituting \eqref{EqJan5qt13h51in2026inSophiaD} in \eqref{EqJan5at11h59in2026inSophia} yields 
\begin{IEEEeqnarray}{rcl}
\nonumber
\label{EqJan5at14h31in2026inSophiaA}
 \RND{P^{\autoparent{Q_{k}, \lambda_{k}}}_{\vect{\Theta}_k | \vect{Z}_k = \vect{z}_{k}}}{Q_{1}}\autoparent{\vect{\theta}}
 &=&
 \frac{\exp\autoparent{ -\sum_{i=1}^{k-1}~\frac{1}{\lambda_i}~\EmpRisk{i}{\vect{z}_i}{\vect{\theta}}}}{\int \exp\autoparent{ -\sum_{i=1}^{k-1}~\frac{1}{\lambda_i}~\EmpRisk{i}{\vect{z}_i}{\vect{\nu}}} \mathrm{d}Q_{1}\autoparent{\vect{\nu}}} \\
 && \RND{P^{\autoparent{Q_1, \lambda_k}}_{\vect{\Theta}_k | \vect{Z}_k = \vect{z}_{k}}}{Q_1}\autoparent{\vect{\theta}} \exp \autoparent{C_{k}}\\
 \label{EqJan5at14h31in2026inSophiaB}
 \nonumber
 &=&  \frac{\exp\autoparent{ -\sum_{i=1}^{k}~\frac{1}{\lambda_i}~\EmpRisk{i}{\vect{z}_i}{\vect{\theta}}}}{\int \exp\autoparent{ -\sum_{i=1}^{k-1}~\frac{1}{\lambda_i}~\EmpRisk{i}{\vect{z}_i}{\vect{\nu}}} \mathrm{d}Q_{1}\autoparent{\vect{\nu}}}\\
&&
 \frac{ \int \exp\autoparent{-\sum_{i=1}^{k-1}~\frac{1}{\lambda_i}~\EmpRiskk{\vect{z}}{\vect{\nu}}{i}}\mathrm{d} Q_1(\vect{\nu})
}{\int \exp\autoparent{ -\sum_{j=1}^{k}~\frac{1}{\lambda_j}~\EmpRiskk{\vect{z}}{\vect{\nu}^\prime}{j}} \mathrm{d} Q_1(\vect{\nu}^\prime)}\squeezeequ\\
\label{EqJan5at14h31in2026inSophiaC}
&=&  \frac{\exp\autoparent{ -\sum_{i=1}^{k}~\frac{1}{\lambda_i}~\EmpRisk{i}{\vect{z}_i}{\vect{\theta}}}}{\int \exp\autoparent{ -\sum_{i=1}^{k}~\frac{1}{\lambda_i}~\EmpRisk{i}{\vect{z}_i}{\vect{\nu}^\prime}} \mathrm{d}Q_{1}\autoparent{\vect{\nu}^\prime}},\squeezeequ
\end{IEEEeqnarray}
where \eqref{EqJan5at14h31in2026inSophiaB} follows from \eqref{EqDec27at11h22in2025inMadrid}; and \eqref{EqJan5at14h31in2026inSophiaC} yields \eqref{EqDec30at11h25in2025Antibes}. This completes the proof. 

\section{Conclusions and Final Remarks}\label{Conclusion}
\enlargethispage{-0.10in}
This work establishes that, in a decentralized machine learning scenario, an appropriate choice of reference measures $(Q_1,Q_2,\ldots,Q_K)$ and regularization factors $(\lambda_1,\lambda_2,\ldots,\lambda_K)$ allows guaranteeing the same performance as a centralized system in which all training datasets are aggregated and jointly available.
The construction of such regularization factors is rather simple. The regularization factor of client~$k$ shall be the product of a strictly positive real (common to all clients) and the ratio of the sizes of the training dataset of client~$k$ and the aggregated training dataset.
The reference measure of client $k$, $k >1$, is the Gibbs measure (Gibbs algorithm) from which client~$k-1$ samples its models. The first client uses a given reference $Q_1$.
Under such a choice, client~$K$ obtains a Gibbs probability measure that solves the ERM-RER problem, with respect to the aggregated dataset, within a neighborhood of~$Q_1$. Via backward dissemination, all clients obtain the same probability measure (algorithm).
This choice of reference measures induces a nested structure whose construction requires transmitting $K-1$ probability measures with common support~\cite[Lemma~3]{perlaza2024empirical}.
Several practical challenges arise from this requirement.
A main limitation is finite-rate communication, possibly under delay constraints, which implies that the probability measure transmitted by client~$k$ may be received by client~$k+1$ with distortion.
Characterizing the impact of such distortions on the nested construction remains an open problem and is not addressed here.
A further limitation is the potentially large support of the involved Gibbs probability measures, which can make the communication requirement comparable to transmitting the training datasets.
Nonetheless, the transmission of a probability measure is more privacy-preserving than the actual transmission of training datasets.
%

\IEEEtriggeratref{12}
\bibliographystyle{\Latexfilepath/Bibliography/IEEEtranlink}
\bibliography{\Latexfilepath/Bibliography/Merged_BERMUDEZ_ref.bib}
 
\appendices
\end{document}